\documentclass{article}

\usepackage[preprint]{neurips_2019}

\usepackage[utf8]{inputenc} 
\usepackage[T1]{fontenc}    
\usepackage{hyperref}       
\usepackage{url}            
\usepackage{booktabs}       
\usepackage{amsfonts}       
\usepackage{nicefrac}       
\usepackage{microtype}      
\usepackage{amsmath}
\usepackage{changepage}
\usepackage{amssymb}
\usepackage{graphicx}%
\usepackage{xcolor}
\usepackage{url}
\usepackage{amsmath}
\usepackage{alg}
\usepackage{wrapfig}
\usepackage{graphicx}
\usepackage{amsmath}
\usepackage{amsthm}
\usepackage{amssymb}
\usepackage{comment}
\usepackage{subcaption}
\usepackage{caption}
\newtheorem{theorem}{Theorem}

\newtheorem{claim}{Claim}

\title{Understanding and Improving Transformer\\ From a Multi-Particle Dynamic System Point of View}

\author{
    Yiping Lu$^{1,}$\thanks{Equal contribution. Works done while interning at Microsoft Research Asia.} \quad Zhuohan Li$^{1,}$\footnotemark[1] \quad Di He$^{1}$ \quad Zhiqing Sun$^{1}$ \\ {\bf Bin Dong$^{1}$ \quad Tao Qin$^{2}$ \quad  Liwei Wang$^{1}$ \quad Tie-Yan Liu$^{2}$} \\
  $^1$Peking University\qquad$^2$Microsoft Research \\
  \texttt{\{luyiping9712,lizhouhan,di\_he,zhiqing\}@pku.edu.cn} \\ \texttt{dongbin@math.pku.edu.cn\quad wanglw@pku.edu.cn} \\
  \texttt{\{taoqin, tyliu\}@microsoft.com}
}

\begin{document}

\maketitle

\begin{abstract}
The Transformer architecture is widely used in natural language processing. Despite its success, the design principle of the Transformer remains elusive. In this paper, we provide a novel perspective towards understanding the architecture: we show that the Transformer can be mathematically interpreted as a \emph{numerical Ordinary Differential Equation (ODE) solver for a convection-diffusion equation in a multi-particle dynamic system}. In particular, how words in a sentence are abstracted into contexts by passing through the layers of the Transformer can be interpreted as approximating multiple particles' movement in the space using the Lie-Trotter splitting scheme and the Euler's method. Given this ODE's perspective, the rich literature of numerical analysis can be brought to guide us in designing effective structures beyond the Transformer. As an example, we propose to replace the Lie-Trotter splitting scheme by the Strang-Marchuk splitting scheme, a scheme that is more commonly used and with much lower local truncation errors. The Strang-Marchuk splitting scheme suggests that the self-attention and position-wise feed-forward network (FFN) sub-layers should not be treated equally. Instead, in each layer, two position-wise FFN sub-layers should be used, and the self-attention sub-layer is placed in between. This leads to a brand new architecture. Such an FFN-attention-FFN layer is ``Macaron-like", and thus we call the network with this new architecture the \emph{Macaron} Net. Through extensive experiments, we show that the Macaron Net is superior to the Transformer on both supervised and unsupervised learning tasks. The reproducible codes and pretrained models can be found at \url{https://github.com/zhuohan123/macaron-net}

\end{abstract}

\section{Introduction}

The Transformer is one of the most commonly used neural network architectures in natural language processing.  Variants of the Transformer have achieved state-of-the-art performance in many tasks including language modeling \cite{dai2019transformer,al2018character} and machine translation \cite{vaswani2017attention,dehghani2018universal,edunov2018understanding}. Transformer-based unsupervised pre-trained models also show impressive performance in many downstream tasks \cite{radford2019language,devlin2018bert,peters2018deep}. 

The Transformer architecture is mainly built by stacking layers, each of which consists of two sub-layers with residual connections: the self-attention sub-layer and the position-wise feed-forward network (FFN) sub-layer.  For a given sentence, the self-attention sub-layer considers the semantics and dependencies of words at different positions and uses that information to capture the internal structure and representations of the sentence. The position-wise FFN sub-layer is applied to each position separately and identically to encode context at each position into higher-level representations. Although the Transformer has demonstrated promising results in many tasks, its design principle is not fully understood, and thus the strength of the architecture is not fully exploited. As far as we know, there is little work studying the foundation of the Transformer or different design choices.

\begin{wrapfigure}[15]{r}{.49\textwidth}
    \centering
    \vspace{-12pt}
    \includegraphics[width=.48\textwidth]{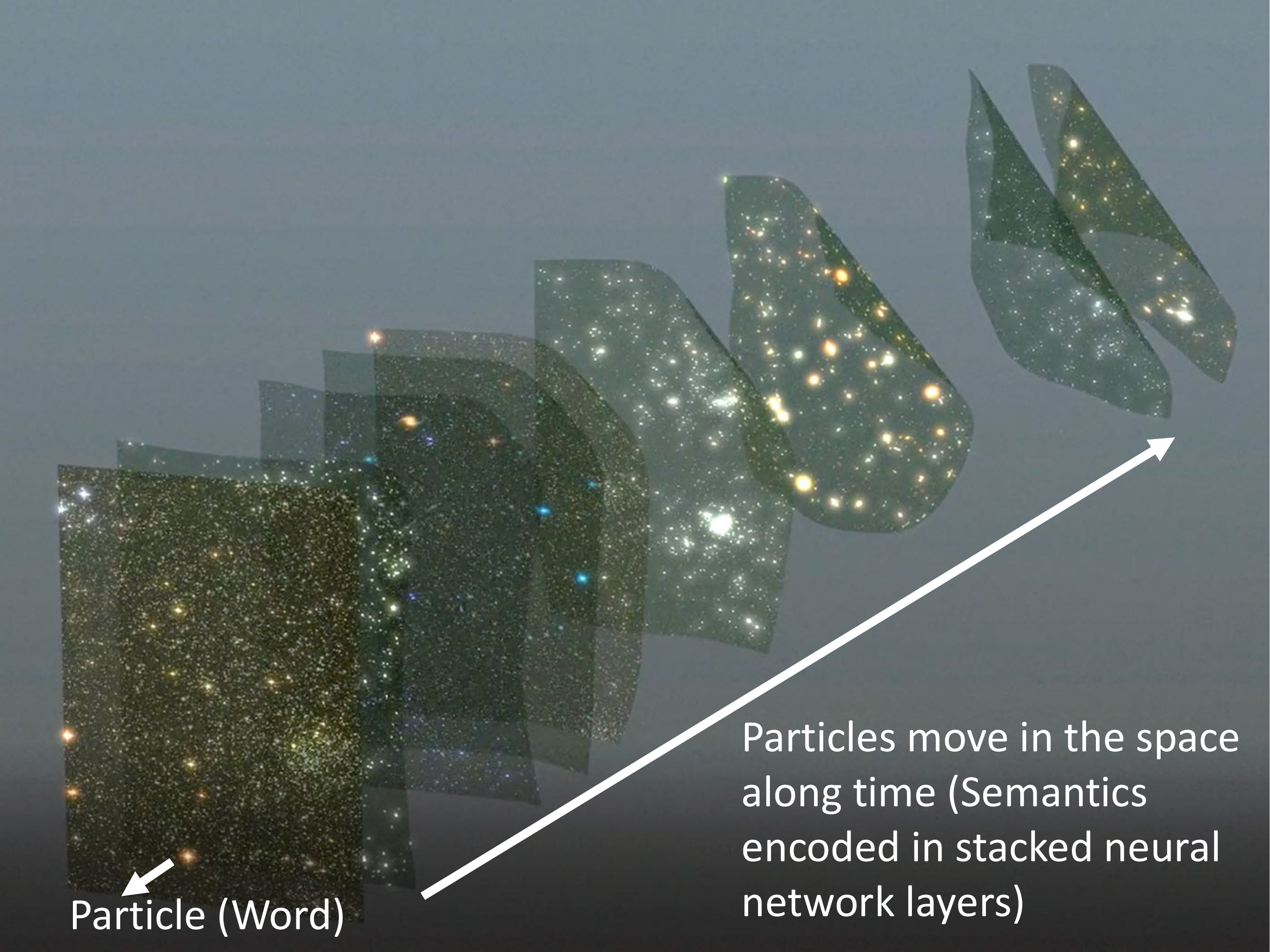}
    \vspace{2pt}
    \caption{Physical interpretation of Transformer.}
    \label{fig:art}
\end{wrapfigure}

In this paper, we provide a novel perspective towards understanding the architecture. In particular, we are the first to show that the Transformer architecture is inherently related to Multi-Particle Dynamic System (MPDS) in physics. MPDS is a well-established research field which aims at modeling how a collection of particles move in the space using differential equations \cite{moulton2012introduction}. In MPDS, the behavior of each particle is usually modeled by two factors separately. The first factor is the convection which concerns the mechanism of each particle regardless of other particles in the system, and the second factor is the diffusion which models the movement of the particle resulting from other particles in the system. 

Inspired by the relationship between the ODE and neural networks \cite{lu2017beyond, chen2018neural}, we first show that the Transformer layers can be naturally interpreted as a numerical ODE solver for a first-order convection-diffusion equation in MPDS. To be more specific, the self-attention sub-layer, which transforms the semantics at one position by attending over all other positions, corresponds to the diffusion term; The position-wise FFN sub-layer, which is applied to each position separately and identically, corresponds to the convection term. The number of stacked layers in the Transformer corresponds to the time dimension in ODE. In this way, the stack of self-attention sub-layers and position-wise FFN sub-layers with residual connections can be viewed as solving the ODE problem numerically using the Lie-Trotter splitting scheme \cite{geiser2009decomposition} and the Euler's method \cite{ascher1998computer}. By this interpretation, we have a novel understanding of learning contextual representations of a sentence using the Transformer: the feature (a.k.a, embedding) of words in a sequence can be considered as the initial positions of a collection of particles, and the latent representations abstracted in stacked Transformer layers can be viewed as the location of particles moving in a high-dimensional space at different time points.

Such an interpretation not only provides a new perspective on the Transformer but also inspires us to design new structures by leveraging the rich literature of numerical analysis. The Lie-Trotter splitting scheme is simple but not accurate and often leads to high approximation error \cite{geiser2009decomposition}. The Strang-Marchuk splitting scheme \cite{strang1968construction} is developed to reduce the approximation error by a simple modification to the Lie-Trotter splitting scheme and is theoretically more accurate. Mapped to neural network design, the Strang-Marchuk splitting scheme suggests that there should be three sub-layers: two position-wise feed-forward sub-layers with half-step residual connections and one self-attention sub-layer placed in between with a full-step residual connection. By doing so, the stacked layers will be more accurate from the ODE's perspective and will lead to better performance in deep learning. As the FFN-attention-FFN layer is ``Macaron-like", we call it \emph{Macaron layer} and call the network composed of Macaron layers the \emph{Macaron Net}. 

We conduct extensive experiments on both supervised and unsupervised learning tasks. For each task, we replace Transformer layers by Macaron layers and keep the number of parameters to be the same. Experiments show that the Macaron Net can achieve higher accuracy than the Transformer on all tasks which, in a way, is consistent with the ODE theory. 

\section{Background}
\subsection{Relationship Between Neural Networks and ODE}
\label{sec:neural-ode}
Recently, there are extensive studies to bridge deep neural networks with ordinary differential equations \cite{weinan2017proposal,lu2017beyond,haber2017stable,chen2018neural,zhang2018dynamically,sonoda2019transport,thorpe2018deep}. We here present a brief introduction to such a relationship and discuss how previous works borrow powerful tools from numerical analysis to help deep neural network design. 

A first-order ODE problem is usually defined as to solve the equation (i.e., calculate $x(t)$ for any $t$) which satisfies the following first-order derivative and the initial condition:
\begin{eqnarray}
\frac{\mathrm{d}x(t)}{\mathrm{d}t}=f(x,t),\quad x(t_0)=w, \label{equ:first-order-ode}
\end{eqnarray}
in which $x(t)\in \mathbb{R}^d$ for all $t\geq t_0$. ODEs usually have physical interpretations. For example, $x(t)$ can be considered as the location of a particle moving in the $d$-dimensional space and the first order time derivative can be considered as the velocity of the particle. 

Usually there is no analytic solution to Eqn (\ref{equ:first-order-ode}) and the problem has to be solved numerically. The simplest numerical ODE solver is the Euler's method \cite{ascher1998computer}. The Euler's method discretizes the time derivative $\frac{\mathrm{d}x(t)}{\mathrm{d}t}$ by its first-order approximation $\frac{x(t_2)-x(t_1)}{t_2-t_1}\approx f(x(t_1),t_1)$. By doing so, for the fixed time horizon $T =t_0+\gamma L$, we can estimate $x(T)$ from $x_0 \doteq x(t_0)$ by sequentially estimating $x_{l+1} \doteq x(t_{l+1})$  using \begin{eqnarray}
x_{l+1} =  x_l+\gamma f(x_l, t_l)
\end{eqnarray}
where $l=0,\cdots,L-1$, $t_l = t_0 + \gamma l$ is the time point corresponds to $x_l$, and $\gamma = (T - t_0)/L$ is the step size. As we can see, this is mathematically equivalent to the ResNet architecture \cite{lu2017beyond,chen2018neural}: The function $\gamma f(x_l,t_l)$ can be considered as a neural-network block, and the second argument $t_l$ in the function indicates the set of parameters in the $l$-th layer. The simple temporal discretization by Euler's method naturally leads to the residual connection.  

Observing such a strong relationship, researchers use ODE theory to explain and improve the neural network architectures mainly designed for computer vision tasks. \cite{lu2017beyond,chen2018neural} show any parametric ODE solver can be viewed as a deep residual network (probably with infinite layers), and the parameters in the ODE can be optimized through backpropagation. Recent works discover that new neural networks inspired by sophisticated numerical ODE solvers can lead to better performance. For example, \cite{zhu2018convolutional} uses a high-precision Runge-Kutta method to design a neural network, and the new architecture achieves higher accuracy. \cite{haber2017stable} uses a leap-frog method to construct a reversible neural network. \cite{liao2016bridging,chang2019antisymmetricrnn} try to understand recurrent neural networks from the ODE's perspective, and \cite{tao2018nonlocal} uses non-local differential equations to model non-local neural networks. 

\subsection{Transformer}
The Transformer architecture is usually developed by stacking Transformer layers \cite{vaswani2017attention, devlin2018bert}. A Transformer layer operates on a sequence of vectors and outputs a new sequence of the same shape. The computation inside a layer is decomposed into two steps: the vectors first pass through a (multi-head) self-attention sub-layer and the output will be further put into a position-wise feed-forward network sub-layer. Residual connection \cite{he2016identity} and layer normalization \cite{lei2016layer} are employed for both sub-layers. The visualization of a Transformer layer is shown in Figure \ref{fig:pipeline}(a) and the two sub-layers are defined as below.

\paragraph{Self-attention sub-layer}
The attention mechanism can be formulated as querying a dictionary with key-value pairs \citep{vaswani2017attention}, e.g., $\text{Attention}(Q,K,V)=\text{softmax}({QK^T}/{\sqrt{d_{\mathit{model}}}})\cdot V, $
where $d_\mathit{model}$ is the dimensionality of the hidden representations and $Q$ (Query), $K$ (Key), $V$ (Value) are specified as the hidden representations of the previous layer in the so-called \emph{self-attention} sub-layers in the Transformer architecture. The multi-head variant of attention allows the model to jointly attend to information from different representation subspaces, and is defined as
\begin{align}
\text{Multi-head}(Q,K,V) ={}& \text{Concat} (\text{head}_1,\cdots,\text{head}_H)W^O,\\
\text{head}_k ={}& \text{Attention}(QW_k^Q, KW_k^K,VW_k^V),
\end{align}
where $W_k^Q\in \mathbb{R}^{d_\mathit{model}\times d_K}, W_k^K\in \mathbb{R}^{d_\mathit{model}\times d_K}, W_k^V\in \mathbb{R}^{d_\mathit{model}\times d_V},$ and $W^O\in \mathbb{R}^{Hd_V \times d_\mathit{model}}$ are project parameter matrices, $H$ is the number of heads, and $d_K$ and $d_V$ are the dimensionalities of Key and Value. 

\paragraph{Position-wise FFN sub-layer} 
In addition to the self-attention sub-layer, each Transformer layer also contains a fully connected feed-forward network, which is applied to each position separately and identically. This feed-forward network consists of two linear transformations with an activation function $\sigma$ in between. Specially, given vectors $h_1, \ldots, h_n$, a position-wise FFN sub-layer transforms each $h_i$ as $\text{FFN}(h_i) = \sigma(h_iW_1+b_1) W_2 + b_2$, where $W_1, W_2, b_1$ and $b_2$ are parameters.

In this paper, we take the first attempt to provide an understanding of the feature extraction process in natural language processing from the ODE's viewpoint. As discussed in Section~\ref{sec:neural-ode}, several works interpret the standard ResNet using the ODE theory. However, we found this interpretation cannot be directly applied to the Transformer architecture. First, different from vision applications whose size of the input (e.g., an image) is usually predefined and fixed, the input (e.g., a sentence) in natural language processing is always of variable length, which makes the single-particle ODE formulation used in previous works not applicable. Second, the Transformer layer contains very distinct sub-layers. The self-attention sub-layer takes the information from all positions as input while the position-wise feed-forward layer is applied to each position separately. How to interpret these heterogeneous components by ODE is also not covered by previous works \cite{tao2018nonlocal,chen2018neural}.

\section{Reformulate Transformer Layers as an ODE Solver for Multi-Particle Dynamic System}
In this section, we first introduce the general form of differential equations in MPDS and then reformulate the stacked Transformer layers to show they form a numerical ODE solver for a specific problem. After that, we use advanced methods in the ODE theory to design new architectures.  

\subsection{Multi-Particle ODE and Its Numerical Solver}
\label{sec:mpode}
Understanding the dynamics of multiple particles' movements in space is one of the important problems in physics, especially in fluid mechanics and astrophysics \cite{moulton2012introduction}. The behavior of each particle is usually modeled by two factors: The first factor concerns about the mechanism of its movement regardless of other particles, e.g., caused by an external force outside of the system, which is usually referred to as the convection; The second factor concerns about the movement resulting from other particles, which is usually referred to as the diffusion. Mathematically, assume there are $n$ particles in $d$-dimensional space. Denote $x_i(t)\in \mathbb{R}^d$ as the location of $i$-th particle at time $t$. The dynamics of particle $i$ can be formulated as  
\begin{align}
\label{equ:trans}
\frac{\mathrm{d}x_i(t)}{\mathrm{d}t} ={}& F(x_i(t),[x_1(t),\cdots,x_n(t)],t) + G(x_i(t),t) ,\nonumber\\
x_i(t_0)={}& w_i, \quad i=1,\ldots,n.
\end{align}
Function $F(x_i(t),[x_1(t),\cdots,x_n(t)],t)$ represents the diffusion term which characterizes the interaction between the particles. $G(x,t)$ is a function which takes a location $x$ and time $t$ as input and represents the convection term. 

\paragraph{Splitting schemes}
As we can see, there are two coupled terms in the right-hand side of Eqn (\ref{equ:trans}) describing different physical phenomena. Numerical methods of directly solving such ODEs can be complicated. The splitting method is a prevailing way of solving such coupled differential equations that can be decomposed into a sum of differential operators \cite{mclachlan2002splitting}. Furthermore, splitting convection from diffusion is quite standard for many convection-diffusion equations \cite{glowinski2017splitting,geiser2009decomposition}. The 
Lie-Trotter splitting scheme \cite{geiser2009decomposition} is the simplest splitting method. It splits the right-hand side of Eqn (\ref{equ:trans}) into function $F(\cdot)$ and $G(\cdot)$ and solves the individual dynamics alternatively. More precisely, to compute $x_i(t+\gamma)$ from $x_i(t)$, the Lie-Trotter splitting scheme with the Euler's method reads as
\begin{align}
    \tilde x_i(t)={}&x_i(t)+\gamma F(x_i(t),[x_1(t),x_2(t),\cdots, x_n(t)],t),\label{equ:split1}\\
    x_i(t+\gamma)={}&\tilde x_i(t)+\gamma G(\tilde x_i(t),t).
    \label{equ:split2}
\end{align}
From time $t$ to time $t+\gamma$, the Lie-Trotter splitting method first solves the ODE with respect to $F(\cdot)$ and acquire an intermediate location $\tilde x_i(t)$. Then, starting from $\tilde x_i(t)$ , it solves the second ODE with respect to $G(\cdot)$ to obtain $x_i(t+\gamma)$.

\subsection{Physical interpretation of the Transformer}

We reformulate the two sub-layers of the Transformer in order to match its form with the ODE described above. Denote $x_l = (x_{l, 1},\ldots,x_{l, n})$ as the input to the $l$-th Transformer layer, where $n$ is the sequence length and $x_{l, i}$ is a real-valued vector in $\mathbb{R}^d$ for any $i$. 

\paragraph{Reformulation of the self-attention sub-layer} Denote $\tilde x_{l, i}$ as the output of the (multi-head) self-attention sub-layer at position $i$ with residual connections. The computation of $\tilde x_{l, i}$ can be written as
\begin{align}
\label{equ:mhselfatt}
\tilde x_{l, i} ={}& x_{l, i} + \text{Concat}\left(\text{head}_1,...,\text{head}_H\right)W^{O, l}, \\
\text{where head}_k ={}& \sum_{j=1}^n \alpha_{ij}^{(k)}[x_{l, j}W^{V,l}_k] = \sum_{j=1}^n \left(\frac{\exp(e_{ij}^{(k)})}{\sum_{q=1}^n\exp(e_{iq}^{(k)})}\right)[x_{l, j}W^{V,l}_k],
\end{align}
and $e_{ij}^{(k)}$ is computed as the dot product of input $x_{l, i}$ and $x_{l, j}$ with linear projection matrices $W^{Q,l}_k$ and $W^{K,l}_k$, i.e., $e_{ij}^{(k)}=d_{\mathit{model}}^{-1/2}\cdot (x_{l, i}W^{Q,l}_k)(x_{l, j}W^{K,l}_k)^T$. Considering $\alpha_{ij}^{(k)}$ as a normalized value of the pair-wise dot product $e_{ij}^{(k)}$ over $j$,  we can generally reformulate Eqn (\ref{equ:mhselfatt}) as
\begin{eqnarray}
\label{equ:selfatt2ode}
\tilde x_{l, i}=x_{l, i}+\text{MultiHeadAtt}_{W_{\mathit{att}}^{l}}(x_{l, i},[ x_{l, 1}, x_{l, 2},\cdots, x_{l, n}]), 
\end{eqnarray}
where $W_{att}^{l}$ denotes all trainable parameters in the $l$-th self-attention sub-layer.
\paragraph{Reformulation of the position-wise FFN sub-layer} Next, $\tilde x_{l, i}$ is put into the position-wise FFN sub-layer with residual connections and output $x_{l+1, i}$. The computation of $x_{l+1, i}$ can be written as
\begin{eqnarray}
\label{equ:ffn2ode}
x_{l+1, i}=\tilde x_{l, i}+\text{FFN}_{W_{\mathit{ffn}}^{l}}(\tilde x_{l, i}),
\end{eqnarray}
where $W_{\mathit{ffn}}^{l}$ denotes all trainable parameters in the $l$-th position-wise FFN sub-layer.
\paragraph{Reformulation of Transformer layers} Combining Eqn (\ref{equ:selfatt2ode}) and (\ref{equ:ffn2ode}), we reformulate the Transformer layers\footnote{Layer normalization is sometimes applied to the sub-layers but recent work \cite{zhang2019fixup} shows that the normalization trick is not essential and can be removed. One can still readily check that the reformulation (Eqn (\ref{equ:transformer2ode1}) and (\ref{equ:transformer2ode2})) still holds with layer normalization.} as
\begin{align}
\label{equ:transformer2ode1}
\tilde x_{l, i}={}&x_{l, i}+\text{MultiHeadAtt}_{W_{\mathit{att}}^{l}}(x_{l, i},[ x_{l, 1}, x_{l, 2},\cdots, x_{l, n}]),\\
\label{equ:transformer2ode2}
x_{l+1, i}={}&\tilde x_{l, i}+\text{FFN}_{W_{\mathit{ffn}}^{l}}(\tilde x_{l, i}).
\end{align}
We can see that the Transformer layers (Eqn  (\ref{equ:transformer2ode1}-\ref{equ:transformer2ode2}))  resemble the multi-particle ODE solver in Section~\ref{sec:mpode} (Eqn (\ref{equ:split1}-\ref{equ:split2})). Indeed, we can formally establish the link between the ODE solver with splitting scheme and stacked Transformer layers as below.
\begin{claim} Define $\gamma F^*(x_{l, i},[x_{l, 1}, \cdots, x_{l, n}], t_l)=\text{MultiHeadAtt}_{W_{\mathit{att}}^{l}}(x_{l, i},[ x_{l, 1}, \cdots, x_{l, n}])$ and $\gamma G^*(x_{l, i}, t_l) = \text{FFN}_{W_{\mathit{ffn}}^{l}}(x_{l, i})$. The Transformer can be viewed as a numerical ODE solver using Lie-Trotter splitting scheme and the Euler's method (with time step $\gamma$) for Eqn (\ref{equ:trans}) with $F^*$ and $G^*$.
\end{claim}

The above observation grants a physical interpretation of natural language processing and provides a new perspective on the Transformer architecture. First, this perspective provides a unified view of the heterogeneous components in the Transformer. The self-attention sub-layer is viewed as a diffusion term which characterizes the particle interactions while the position-wise feed-forward networks is viewed as a convection term. The two terms together naturally form the convection-diffusion equation in physics. Second, this interpretation advances our understanding of the latent representations of language through the Transformer. Viewing the feature (a.k.a., embedding) of words in a sequence as the initial position of particles, we can interpret the latent representations of the sentence abstracted by the Transformer as particles moving in a high-dimensional space as demonstrated in Figure~\ref{fig:art} \cite{art}. 
 
\begin{figure}
    \centering
    \includegraphics[width=.75\textwidth]{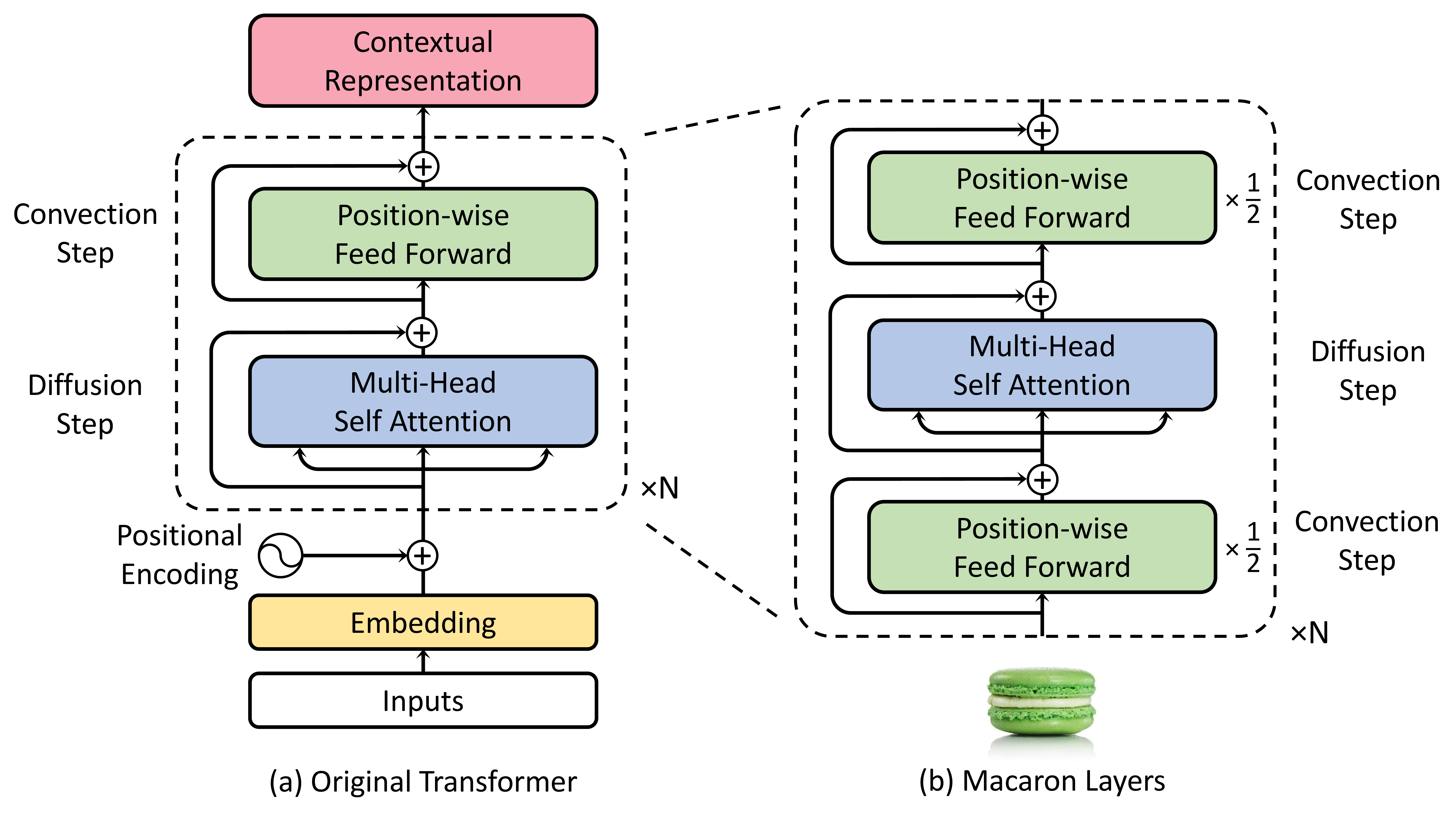}
\caption{The Transformer and our Macaron architectures.}    \label{fig:pipeline}
\end{figure}
 
\subsection{Improving Transformer Via Strang-Marchuk Splitting Scheme}
In the previous subsection, we have successfully mapped the Transformer architecture to a numerical ODE solver for MPDS. However, we would like to point out that one of the key components in this ODE solver, the Lie-Trotter splitting scheme, is the simplest one but has relatively high errors. In this subsection, we incorporate one of the most popular and widely used splitting scheme \cite{geiser2009decomposition}, the Strang-Marchuk splitting scheme, into the design of the neural networks.

The Lie-Trotter splitting scheme solves the dynamics of $F(\cdot)$ and $G(\cdot)$ alternatively and exclusively in that order. This inevitably brings bias and leads to higher local truncation errors \cite{geiser2009decomposition}. To mitigate the bias,  we use a simple modification to the Lie-Trotter splitting scheme by dividing the one-step numerical solver for $G(\cdot)$ into two \emph{half-steps}: we put one half-step before solving $F(\cdot)$ and put the other half-step after solving $F(\cdot)$. This modified splitting scheme is known as the Strang-Marchuk splitting scheme \cite{strang1968construction}. Mathematically, to compute $x_i(t+\gamma)$ from $x_i(t)$, the Strang-Marchuk splitting scheme reads as
\begin{align}
        \tilde x_i(t)={}&x_i(t)+\frac{\gamma}{2} G(x_i(t),t),\label{Equ:111split0}\\
\hat x_i(t)={}&\tilde x_i(t)+\gamma F(\tilde x_i(t),[\tilde x_1(t),\tilde x_2(t),\cdots,\tilde x_n(t)],t),\label{Equ:111split1}\\
    x_i(t+\gamma)={}&\hat x_i(t)+\frac{\gamma}{2} G\left(\hat x_i(t),t+\frac{\gamma}{2}\right).
    \label{Equ:111split2}
\end{align}

The Strang-Marchuk splitting scheme enjoys higher-order accuracy than the Lie-Trotter splitting scheme \cite{bobylev2001error} in terms of the \emph{local truncation error} \cite{ascher1998computer}, which measures the per-step distance between the true solution and the approximated solution using numerical schemes. Mathematically, for a differential equation $\frac{\mathrm{d}x(t)}{\mathrm{d}t}=f(x,t)$ and a numerical scheme $\mathcal{A}$, the local truncation error of numerical scheme $\mathcal{A}$ is defined as $\tau = x(t+\gamma)-\mathcal{A}(x(t),\gamma)$. For example, when $\mathcal{A}$ is the Euler's method, $\tau_{Euler} = x(t+\gamma)-x(t)-\gamma f(x(t),t)$. The order of local truncation error of the two schemes has been studied in \cite{bobylev2001error}, as shown in the following theorem. 
\begin{theorem} \cite{bobylev2001error}
The local truncation error of the Lie-Trotter splitting scheme is \textbf{second-order} ($O(\gamma^2)$) and the local truncation error of the Strang-Marchuk splitting scheme is \textbf{third-order} ($O(\gamma^3)$). 
\end{theorem}
For completeness, we provide the formal theorem with proof in Appendix~\ref{sec:proof}. We can see from Eqn (\ref{Equ:111split0}-\ref{Equ:111split2}) that the Strang-Marchuk splitting scheme uses a three-step process to solve the ODE. Mapped to neural network design, the Strang-Marchuk splitting scheme (together with the Euler's method) suggests there should also be three sub-layers instead of the two sub-layers in Transformer. By replacing function $\gamma F$ and $\gamma G$ by MultiHeadAtt and FFN, we have
\begin{align}
\tilde x_{l, i}={}& x_{l, i} + \frac{1}{2}\text{FFN}_{W_{\mathit{ffn}}^{l, \mathit{down}}}(x_{l, i}),\label{equ:oreo2ode0}\\
\hat x_{l, i}={}&\tilde x_{l, i}+\text{MultiHeadAtt}_{W_{\mathit{att}}^{l}}(\tilde x_{l, i},[\tilde x_{l, 1},\tilde x_{l, 2},\cdots,\tilde x_{l, n}]),\label{equ:oreo2ode1}\\
x_{l+1, i}={}&\hat x_{l, i}+\frac{1}{2}\text{FFN}_{W_{\textit{ffn}}^{l,\textit{up}}}(\hat x_{l, i}).\label{equ:oreo2ode2}
\end{align}
From Eqn (\ref{equ:oreo2ode0}-\ref{equ:oreo2ode2}), we can see that the new layer composes of three sub-layers. Each hidden vector at different positions will first pass through the first position-wise FFN sub-layer with a half-step residual connection (``$\frac{1}{2}$'' in Eqn (\ref{equ:oreo2ode0})), and then the output vectors will be feed into a self-attention sub-layer. In the last step, the vectors outputted from the self-attention sub-layer will be put into the second position-wise FFN sub-layer with a half-step residual connection. Since the FFN-attention-FFN structure is ``Macaron''-like, we call the layer as \emph{Macaron} layer and call the network using Macaron layers as Macaron Net, as shown in Figure~\ref{fig:pipeline}(b). Previous works \cite{lu2017beyond,zhu2018convolutional} have successfully demonstrated that the neural network architectures inspired from higher-order accurate numerical ODE solvers will lead to better results in deep learning and we believe the Macaron Net can achieve better performance on practical natural language processing applications than the Transformer.

\begin{table}
\centering
\caption{Translation performance (BLEU) on IWSLT14 De-En and WMT14 En-De  testsets.}
\vspace{2pt}
\label{tbl:translation}
\centerline{
\begin{tabular}{lccc}
\toprule
& \textbf{IWSLT14 De-En} & \multicolumn{2}{c}{\textbf{WMT14 En-De}}      \\ 
\textbf{Method} & \texttt{small} & \texttt{base} & \texttt{big}  \\
\midrule
Transformer \cite{vaswani2017attention}  & 34.4 & 27.3 & 28.4 \\
Weighted Transformer \cite{ahmed2017weighted} & / & 28.4 & 28.9 \\
Relative Transformer \cite{shaw2018self} & / & 26.8 & 29.2 \\
Universal Transformer \cite{dehghani2018universal} & / & \textbf{28.9} & / \\
Scaling NMT \cite{ott2018scaling} & / & / & 29.3 \\
Dynamic Conv \cite{wu2019pay} & 35.2 & / & 29.7 \\
\midrule
\textbf{Macaron Net} & \textbf{35.4} & \textbf{28.9} & \textbf{30.2} \\
\bottomrule
\end{tabular}   
}
\end{table}

\section{Experiments}
We test our proposed Macaron architectures in both supervised and unsupervised learning setting. For supervised learning setting, we use IWLST14 and WMT14 machine translation datasets. For unsupervised learning setting, we pretrain the model using the same method as in \cite{devlin2018bert} and test the learned model over a set of downstream tasks. Due to space limitations, we put dataset description, model description, hyperparameter configuration into Appendix~\ref{sec:extra-exp}.
\subsection{Experiment Settings}
\paragraph{Machine Translation}
Machine translation is an important application for natural language processing \cite{vaswani2017attention}. We evaluate our methods on two widely used public datasets: IWSLT14 German-to-English (De-En) and WMT14 English-to-German (En-De) dataset. 

For the WMT14 dataset, the basic configurations of the Transformer architecture are the \texttt{base} and the \texttt{big} settings \cite{vaswani2017attention}. Both of them consist of a 6-layer encoder and 6-layer decoder. The size of the hidden nodes and embeddings are set to 512 for \texttt{base} and 1024 for \texttt{big}. The number of heads are 8 for \texttt{base} and 16 for \texttt{big}. Since the IWSLT14 dataset is much smaller than the WMT14 dataset, the \texttt{small} setting is usually used, whose size of hidden states and embeddings is set to 512 and the number of heads is set to 4. For all settings, the dimensionality of the inner-layer of the position-wise FFN is four times of the dimensionality of the hidden states. 

For each setting (\texttt{base}, \texttt{big} and \texttt{small}), we replace all Transformer layers by the Macaron layers\footnote{The translation model is based on the encoder-decoder framework. In the Transformer, the decoder layer has a third sub-layer which performs multi-head attention over the output of the encoder stack (encoder-decoder-attention) and a mask to prevent positions from attending to subsequent positions. In our implementation of Macaron decoder, we also use masks and split the FFN into two sub-layers and thus our decoder layer is (FFN, self-attention, encoder-decoder-attention, and FFN).} and obtain the \texttt{base}, \texttt{big} and \texttt{small} Macaron, each of which contains two position-wise feed-forward sub-layers in a layer. To make a fair comparison, we set the dimensionality of the inner-layer of the two FFN sub-layers in the Macaron layers to two times of the dimensionality of the hidden states. By doing this, the \texttt{base}, \texttt{big} and \texttt{small} Macaron have the same number of parameters as the \texttt{base}, \texttt{big} and \texttt{small} Transformer respectively. 

\paragraph{Unsupervised Pretraining}
BERT \cite{devlin2018bert} is the current state-of-the-art pre-trained contextual representation model based on a multi-layer Transformer encoder architecture and trained by masked language modeling and next-sentence prediction tasks. We compare our proposed Macaron Net with the \texttt{base} setting from the original BERT paper \cite{devlin2018bert}, which consists of 12 Transformer layers. The size of hidden states and embeddings are set to 768, and the number of attention heads is set to 12. Similarly, we replace the Transformer layers in BERT \texttt{base} by the Macaron layers and reduce the dimensionality of the inner-layer of the two FFN sub-layers by half, and thus we keep the number of parameters of our Macaron \texttt{base} same as BERT \texttt{base}. 

\begin{table}[t]
    \centering
    \caption{Test results on the GLUE benchmark (except WNLI).}\vspace{2pt}
\centerline{
\scalebox{0.83}{
\begin{tabular}{lccccccccc}
\toprule
\textbf{Method}    & CoLA          & SST-2         & MRPC      & STS-B     & QQP       & MNLI\footnotesize{-m/mm} & QNLI          & RTE           & \textbf{GLUE} \\ 
\midrule
\multicolumn{2}{l}{\emph{Existing systems}} \\
\midrule
ELMo \cite{peters2018deep}      & 33.6          & 90.4          & 84.4/78.0 & 74.2/72.3	 & 63.1/84.3	  & 74.1/74.5 & 79.8          & 58.9          & 70.0          \\
OpenAI GPT \cite{radford2018improving} & 47.2          & 93.1          & 87.7/83.7 & 85.3/84.8 & 70.1/88.1 & 80.7/80.6 & 87.2          & 69.1 & 76.9          \\
BERT \texttt{base} \cite{devlin2018bert}   & 52.1          & 93.5          & \textbf{88.9}/\textbf{84.8} & 87.1/85.8 & \textbf{71.2}/\textbf{89.2} & 84.6/83.4 & 90.5 & 66.4          & 78.3 \\
\midrule 
\multicolumn{2}{l}{\emph{Our systems}} \\
\midrule
BERT \texttt{base} (ours)   & 52.8          & 92.8          & 87.3/83.0 & 81.2/80.0 & 70.2/88.4 & 84.4/83.7 & 90.4 & 64.9          & 77.4          \\
Macaron Net \texttt{base}   & \textbf{57.6} & \textbf{94.0} & 88.4/84.4 & \textbf{87.5}/\textbf{86.3} & 70.8/89.0 & \textbf{85.4}/\textbf{84.5} & \textbf{91.6}          & \textbf{70.5}          & \textbf{79.7} \\ 
\bottomrule
\end{tabular}
}
}
    \label{tbl:bert}
\end{table}

\subsection{Experiment Results}

\paragraph{Machine Translation}
We use BLEU \cite{papineni2002bleu} as the evaluation measure for machine translation. Following common practice, we use tokenized case-sensitive BLEU and case-insensitive BLEU for WMT14 En-De and IWSLT14 De-En respectively.

The results for machine translation are shown in Table~\ref{tbl:translation}. For the IWSLT14 dataset, our Macaron \texttt{small} outperforms the Transformer \texttt{small} by 1.0 point in terms of BLEU. For the WMT14 dataset, our Macaron \texttt{base} outperforms its Transformer counterpart by 1.6 BLEU points. Furthermore, the performance of our Macaron \texttt{base} model is even better than that of the Transformer \texttt{big} model reported in \cite{vaswani2017attention}, but with much less number of parameters. Our Macaron \texttt{big} outperforms the Transformer \texttt{big} by 1.8 BLEU points. Comparing with other concurrent works, the improvements in our proposed method are still significant.

\paragraph{Unsupervised Pretraining} Following \cite{devlin2018bert}, we evaluate all models by fine-tuning them on 8 downstream tasks in the General Language Understanding Evaluation (GLUE) benchmark \cite{wang2019glue}, including CoLA \cite{warstadt2018neural}, SST-2 \cite{socher2013recursive}, MRPC \cite{dolan2005automatically}, STS-B \cite{cer2017semeval}, QQP \cite{chen2018quora}, MNLI \cite{williams2018broad}, QNLI \cite{rajpurkar2016squad}, and RTE \cite{bentivogli2009fifth}. More details about individual tasks and their evaluation metrics can be found in Appendix~\ref{sec:glue} and \cite{wang2019glue, devlin2018bert}.
To fine-tune the models, we follow the hyperparameter search space in \cite{devlin2018bert} for all downstream tasks, including different batch sizes, learning rates, and numbers of epochs. 

The GLUE results are presented in Table~\ref{tbl:bert}. We present the results of two BERT \texttt{base} models. One is from \cite{devlin2018bert}, and the other is trained using our own data. Due to that some pieces of data used in \cite{devlin2018bert} are no longer freely distributed, the two numbers are slightly different. We can see from the table, our proposed Macaron Net \texttt{base} outperforms all baselines in terms of the general GLUE score. Specifically, given the same dataset, our proposed model outperforms our trained BERT \texttt{base} model in all tasks. Even comparing with the BERT \texttt{base} model in \cite{devlin2018bert}, our model performs better in 6 out of 8 tasks and achieves close performance in the rest 2 tasks. Comparing  our results with original BERT model, we can see that given a fixed number of parameters and a fixed number of attention sub-layers, stacking more FFN sub-layers can get better performance on downstream tasks. 

As a summary, the improvement in both machine translation and GLUE tasks well aligns with the ODE theory and our proposed architecture performs better than the Transformer in real practice.

\section{Conclusion and Future Work}

In this paper, we interpret the Transformer as a numerical ODE solver for a convection-diffusion equation in a multi-particle dynamic system. Specifically, how words in a sentence are abstracted into contexts by passing through the layers of the Transformer can be interpreted as approximating multiple particles' movement in the space using the Lie-Trotter splitting scheme and the Euler's method. By replacing the Lie-Trotter splitting scheme with the more advanced Strang-Marchuk splitting scheme, we obtain a new architecture. The improvements in real applications show the effectiveness of the new model and are consistent with the ODE theory. In the future, we will explore deeper connections between the ODE theory and the Transformer models and use the ODE theory to improve the individual components in the Transformer architecture such as attention modules.

\section*{Acknowledgements}
We would like to thank Shuonan Wu for helpful discussions and Linyuan Gong for the help on unsupervised pretraining experiments.

\bibliographystyle{abbrv}
\bibliography{arxiv}

\newpage
\noindent\LARGE{\textbf{Appendix}}\normalsize
\appendix
\section{Proof of the Theorem}
\label{sec:proof}
\begin{theorem}(\cite{bobylev2001error})
We denote the true solution of equation $\frac{\mathrm{d}x}{\mathrm{d}t}=F(x),x(0)=y_0$\footnote{Since a time-dependent ODE can be formulated as a time-independent ODE by introducing an auxiliary variable \cite{Chicone2007Ordinary}, the theorem here developed for time-independent ODEs can also be applied to time-dependent ODEs without loss of generality.} at time $t$ as $x(t)=S^{t}_F(y_0)$. Simliarily, we can define $S^t_G$ and $S^t_{F+G}$. The local truncation error at $t=0$ of the Lie-Trotter splitting scheme is $S_{F+G}^{\gamma}(y_0)-S_G^{\gamma}[S_F^{\gamma}(y_0)]=\frac{\gamma^2}{2} \{F'(y_0)(y_0)G(y_0)-G'(y_0)(y_0)F(y_0)\}+O(\gamma^3)$ which is \textbf{second-order}. The local truncation error at $t=0$ of the Strang-Marchuk splitting scheme $S_G^{\gamma/2}\{S_F^{\gamma}[S_G^{\gamma/2}(y_0)]\}$ is $S_{F+G}^{\gamma}(y_0)-S_G^{\gamma/2}\{S_F^{\gamma}[S_G^{\gamma/2}(y_0)]\}=O(\gamma^3)$
which is \textbf{third-order}.
\end{theorem}

\begin{proof}
Since
\begin{align*}
S_F^{\gamma}(y_0)={}&y_0+\gamma F(y_0)+\frac{\gamma^2}{2}F'(y_0)F(y_0)+O(\gamma^3), \\
S_G^{\gamma}(y_0)={}&y_0+\gamma G(y_0)+\frac{\gamma^2}{2}G'(y_0)G(y_0)+O(\gamma^3), 
\end{align*}
we have
\begin{align*}
    S_G^{\gamma}[S_F^{\gamma}(y_0)]=S_F^{\gamma}(y_0)+\gamma G(S_F^{\gamma}(y_0))+\frac{\gamma^2}{2}G'(S_F^{\gamma}(y_0))G(S_F^{\gamma}(y_0))+O(\gamma^3).
\end{align*}
At the same time, we have
\begin{align*}
G(y_0+\gamma F(y_0)+O(\gamma^2))={}&G(y_0)+\gamma G'(y_0)F(y_0)+O(\gamma^2), \\ G'(S_F^{\gamma}(y_0))G(S_F^{\gamma}(y_0)) ={}&G'(y_0)G(y_0)+O(\gamma).
\end{align*}
Combine the estimations we have
\begin{align*}
    S_G^{\gamma}[S_F^{\gamma}(y_0)]={}&y_0+\gamma[F(y_0)+G(y_0)] \\ &{}+\frac{\gamma^2}{2}[G'(y_0)G(y_0)+F'(y_0)F(y_0)+2G'(y_0)F(y_0)]+O(\gamma^3).
\end{align*}
As a result, we can estimate the local truncation error of Lie-Trotter splitting scheme as
\begin{align*}
    &S_{F+G}^{\gamma}(y_0)-S_G^{\gamma}[S_F^{\gamma}(y_0)]\\={}&y_0+\gamma (F(y_0)+G(y_0))+\frac{\gamma^2}{2}(F'(y_0)+G'(y_0))(F(y_0)+G(y_0))+O(\gamma^3)\\&-(y_0+\gamma[F(y_0)+G(y_0)]+\frac{\gamma^2}{2}[G'(y_0)G(y_0)+F'(y_0)F(y_0)+2G'(y_0)F(y_0)]+O(\gamma^3))\\
    ={}&\frac{\gamma^2}{2} \{F'(y_0)G(y_0)-G'(y_0)F(y_0)\}+O(\gamma^3).
\end{align*}
To estimate the Strang-Marchuk splitting scheme's local truncation error, we rewrite the Strang-Marchuk splitting scheme as
\[
S_G^{\gamma/2}\{S_F^{\gamma}[S_G^{\gamma/2}(y_0)]\}=
S_G^{\gamma/2}\{S_F^{\gamma/2}\{S_F^{\gamma/2}[S_G^{\gamma/2}(y_0)]\}\}.
\]
From the previous estimation of Lie--Trotter splitting scheme we have
\begin{align*}
S_{F+G}^{\gamma/2}(y_0)-S_G^{\gamma}[S_F^{\gamma}(y_0)]={}&\frac{\gamma^2}{8} \{F'(y_0)G(y_0)-G'(y_0)F(y_0)\}+O(\gamma^3), \\
S_{F+G}^{\gamma/2}(y_0)-S_F^{\gamma}[S_G^{\gamma}(y_0)]={}&\frac{\gamma^2}{8} \{G'(y_0)F(y_0)-F'(y_0)G(y_0)\}+O(\gamma^3).
\end{align*}

Combine the two estimations, we have
\begin{align*}
S_G^{\gamma/2}\{S_F^{\gamma/2}\{S_F^{\gamma/2}[S_G^{\gamma/2}(y_0)]\}\}={}&
S_{F+G}^{\gamma}(y_0)+\frac{\gamma^2}{8} \{F'(y_0)G(y_0)-G'(y_0)F(y_0)\}\\&+\frac{\gamma^2}{8} \{G'(y_0)F(y_0)-F'(y_0)G(y_0)\}+O(\gamma^3)\\
={}&S_{F+G}^{\gamma}(y_0)+O(\gamma^3).
\end{align*}
\end{proof}

\section{Experiment Settings}
\label{sec:extra-exp}
\subsection{Machine Translation}

\paragraph{Dataset}
The training/validation/test sets of the IWSLT14 dataset contain about 153K/7K/7K sentence pairs, respectively. We use a vocabulary of 10K tokens based on a joint source and target byte pair encoding (BPE) \cite{BPE}. For WMT14 dataset, we replicate the setup of \cite{vaswani2017attention}, which contains 4.5M training parallel sentence pairs. Newstest2014 is used as the test set, and Newstest2013 is used as the validation set. The 37K vocabulary for WMT14 is based on a joint source and target BPE factorization.

\paragraph{Model} For the WMT14 dataset, the basic configurations of the Transformer architecture are the \texttt{base} and the \texttt{big} settings \cite{vaswani2017attention}. Both of them consist of a 6-layer encoder and 6-layer decoder. The size of the hidden nodes and embeddings are set to 512 for \texttt{base} and 1024 for \texttt{big}. The number of heads are 8 for \texttt{base} and 16 for \texttt{big}. Since the IWSLT14 dataset is much smaller than the WMT14 dataset, the \texttt{small} setting is usually used, whose size of hidden states and embeddings is set to 512 and the number of heads is set to 4. For all settings, the dimensionality of the inner-layer of the position-wise FFN is four times of the dimensionality of the hidden states. 

For each setting (\texttt{base}, \texttt{big} and \texttt{small}), we replace all Transformer layers by the Macaron layers and obtain the \texttt{base}, \texttt{big} and \texttt{small} Macaron, each of which contains two position-wise feed-forward sub-layers in a layer. The translation model is based on the encoder-decoder framework. In the Transformer, the decoder layer has a third sub-layer which performs multi-head attention over the output of the encoder stack (encoder-decoder-attention) and a mask to prevent positions from attending to subsequent positions. In our implementation of Macaron decoder, we also use masks and split the FFN into two sub-layers and thus our decoder layer is (FFN, self-attention, encoder-decoder-attention and FFN). 

To make a fair comparison, we set the dimensionality of the inner-layer of the two FFN sub-layers in the Macaron layers to two times of the dimensionality of the hidden states. By doing this, the \texttt{base}, \texttt{big} and \texttt{small} Macaron have the same number of parameters as the \texttt{base}, \texttt{big} and \texttt{small} Transformer respectively. 

\paragraph{Optimizer and training} We use the Adam optimizer and follow the optimizer setting and learning rate schedule in \cite{vaswani2017attention}. For the \texttt{big} setting, we enlarge the batch size and learning rate as suggested in \cite{ott2018scaling} to accelerate training. We employ label smoothing of value $\epsilon_\mathit{ls}=0.1$ \citep{szegedy2016rethinking} in all experiments. Models for WMT14/IWSLT14 are trained on 4/1 NVIDIA P40 GPUs respectively. Our code is based on the open-sourced \texttt{fairseq} \cite{gehring2017convs2s} code base in PyTorch toolkit.

\paragraph{Evaluation} We use BLEU\footnote{\url{https://github.com/moses-smt/mosesdecoder/blob/master/scripts/generic/multi-bleu.perl}} \cite{papineni2002bleu} as the evaluation measure for machine translation. Following common practice, we use tokenized case-sensitive BLEU and case-insensitive BLEU for WMT14 En-De and IWSLT14 De-En respectively. During inference, we use beam search with beam size 4 and length penalty 0.6 for WMT14, and beam size 5 and length penalty 1.0 for IWSLT14, following \cite{vaswani2017attention}. 

\subsection{Unsupervised Pretraining}

\paragraph{Pre-training dataset} We follow \cite{devlin2018bert} to use English Wikipedia corpus and BookCorpus for pre-training. As the dataset BookCorpus \cite{moviebook} is no longer freely distributed. We follow the suggestions from \cite{devlin2018bert} to crawl and collect BookCorpus\footnote{\url{https://www.smashwords.com/}} on our own. The concatenation of two datasets includes roughly 3.4B words in total, which is comparable with the data corpus used in \cite{devlin2018bert}. We first segment documents into sentences with Spacy;\footnote{\url{https://spacy.io}} Then, we normalize, lower-case, and tokenize texts using Moses\cite{Koehn2007MosesOS} and apply BPE\cite{BPE}. We randomly split documents into one training set and one validation set. The training-validation ratio for pre-training is 199:1. 

\paragraph{Model} We compare our proposed Macaron Net with the \texttt{base} setting from the original BERT paper \cite{devlin2018bert}, which consists of 12 Transformer layers. The size of hidden states and embeddings are set to 768, and the number of attention heads is set to 12. Similarly, we replace the Transformer layers in BERT \texttt{base} by the Macaron layers and reduce the dimensionality of the inner-layer of the two FFN sub-layers by half, and thus we keep the number of parameters of our Macaron \texttt{base} as the same as BERT \texttt{base}. 

\paragraph{Optimizer and training} We follow \cite{devlin2018bert} to use two tasks to pretrain our model. One task is masked language modeling, which masks some percentage of the input tokens at random,
and then requires the model to predict those masked tokens. Another task is next sentence prediction, which requires the model to predict whether two sentences in a given sentence pair are consecutive. We use the Adam optimizer and and follow the optimizer setting and learning rate schedule in \cite{devlin2018bert} and trained the model on 4 NVIDIA P40 GPUs. 

\subsection{GLUE Dataset}
\label{sec:glue}
We provide a brief description of the tasks in the GLUE benchmark \cite{wang2019glue} and our fine-tuning process on the GLUE datasets.

\paragraph{CoLA} The Corpus of Linguistic Acceptability \cite{warstadt2018neural} consists of English acceptability judgments drawn from books and journal articles on linguistic theory. The task is to predict whether an example is a grammatical English sentence. The performance is evaluated by Matthews correlation coefficient \cite{matthews1975comparison}.

\paragraph{SST-2} The Stanford Sentiment Treebank \cite{socher2013recursive} consists of sentences from movie reviews and human annotations of their sentiment. The task is to predict the sentiment of a given sentence (positive/negative). The performance is evaluated by the test accuracy.

\paragraph{MRPC} The Microsoft Research Paraphrase Corpus \cite{dolan2005automatically} is a corpus of sentence pairs automatically extracted from online news sources, with human annotations for whether the sentences in the pair are semantically equivalent, and the task is to predict the equivalence. The performance is evaluated by both the test accuracy and the test F1.

\paragraph{STS-B} The Semantic Textual Similarity Benchmark \cite{cer2017semeval} is a collection of sentence pairs drawn from news headlines, video and image captions, and natural language inference data. Each pair is human-annotated with a similarity score from 1 to 5; the task is to predict these scores. The performance is evaluated by Pearson and Spearman correlation coefficients.

\paragraph{QQP} The Quora Question Pairs\footnote{\url{https://data.quora.com/First-Quora-Dataset-Release-Question-Pairs}} \cite{chen2018quora} dataset is a collection of question pairs from the community question-answering website Quora. The task is to determine whether a pair of questions are semantically equivalent. The performance is evaluated by both the test accuracy and the test F1.

\paragraph{MNLI} The Multi-Genre Natural Language Inference Corpus \cite{williams2018broad} is a crowdsourced collection of sentence pairs with textual entailment annotations. Given a premise sentence and a hypothesis sentence, the task is to predict whether the premise entails the hypothesis (\emph{entailment
}), contradicts the hypothesis (\emph{contradiction}), or neither (\emph{neutral}). The performance is evaluated by the test accuracy on both \emph{matched} (in-domain) and \emph{mismatched} (cross-domain) sections of the test data.

\paragraph{QNLI} The Question-answering NLI dataset is converted from the Stanford Question Answering Dataset (SQuAD) \cite{rajpurkar2016squad} to a classification task. The performance is evaluated by the test accuracy.

\paragraph{RTE} The Recognizing Textual Entailment (RTE) datasets come from a series of annual textual entailment challenges \cite{bentivogli2009fifth}. The task is to predict whether sentences in a sentence pair are entailment. The performance is evaluated by the test accuracy.

\paragraph{WNLI} The Winograd Schema Challenge \cite{levesque2011winograd} is a reading comprehension task in which a system must read a sentence with a pronoun and select the referent of that pronoun from a list of choices. We follow \cite{devlin2018bert} to skip this task in our experiments, because few previous works do better than predicting the majority class for this task.

\paragraph{Fine-tuning on GLUE tasks} To fine-tune the models, following \cite{devlin2018bert}, we search the optimization hyperparmaters in a search space including different batch sizes (16/32), learning rates (5e-3/3e-5), number of epochs (3/4/5), and a set of different random seeds. We select the model for testing according to their performance on the development set.

\paragraph{Test data} Note that the GLUE dataset distribution does not include the Test labels, and we only made a single GLUE evaluation server submission\footnote{\url{https://gluebenchmark.com}} for each of our models.

\end{document}